\documentclass{article}

\usepackage{microtype}
\usepackage{graphicx}
\usepackage{subcaption}
\usepackage{booktabs} 

\usepackage{hyperref}

\usepackage[accepted]{icml2026}

\usepackage{amsmath}
\usepackage{amssymb}
\usepackage{mathtools}
\usepackage{amsthm}
\usepackage{xspace}

\usepackage[capitalize,noabbrev]{cleveref}

\theoremstyle{plain}
\newtheorem{theorem}{Theorem}[section]

\newtheorem{lemma}[theorem]{Lemma}

\theoremstyle{definition}
\newtheorem{definition}[theorem]{Definition}
\newtheorem{assumption}[theorem]{Assumption}
\theoremstyle{remark}

\usepackage[disable,textsize=tiny]{todonotes}
\newcommand{\method}[0]{\textsc{Cure}\xspace}
\newcommand{\methodlong}[0]{\underline{C}ontext management via \underline{U}ncertainty-aware admission and \underline{R}edundancy-aware \underline{E}viction\xspace}

\icmltitlerunning{Bounded Context Management for Tabular Foundation Models on Stream Learning}

\begin{document}

\twocolumn[
  \icmltitle{Bounded Context Management for \\ Tabular Foundation Models on Stream Learning}
    
  \icmlsetsymbol{equal}{*}

  \begin{icmlauthorlist}
    \icmlauthor{Jinmo Lee}{snu}
    \icmlauthor{Doyun Choi}{snu}
    \icmlauthor{Moongi Choi}{kaist}
    \icmlauthor{Jaemin Yoo}{snu}
  \end{icmlauthorlist}
  
  \icmlaffiliation{snu}{Department of Computer Science and Engineering, Seoul National University, Seoul, Republic of Korea}
  \icmlaffiliation{kaist}{The Kim Jaechul Graduate School of AI, KAIST, Daejeon, Republic of Korea}

  \icmlcorrespondingauthor{Jinmo Lee}{jinmo.lee@snu.ac.kr}
  \icmlcorrespondingauthor{Jaemin Yoo}{jaeminyoo@snu.ac.kr}

  \icmlkeywords{Machine Learning, ICML}

  \vskip 0.3in
]

\printAffiliationsAndNotice{} 

\begin{abstract}
  Tabular stream learning requires predictions on sequentially arriving examples under distribution shift. While standard methods adapt by updating model states, tabular foundation models (TFMs) make predictions conditioned on a labeled context in an in-context manner, making them a natural alternative for stream learning. This shifts the challenge from how to update the model to how to manage the context. We propose a future-information view that yields three practical requirements for context management: preserve recent examples, retain uncertain examples, and remove redundant examples. We instantiate these requirements as \method (\methodlong), a context-managing policy with entropy-gated admission and redundancy-aware eviction. Across seven streams, \method shows up to $27.0\%$ relative improvement over classical stream learners, remains robust across multiple TFM backbones, and ranks first among other policy variants. Code and datasets are available \href{https://github.com/morcellinus/CURE-ICML-FMSD/tree/main}{here}.

\end{abstract}

\section{Introduction}
\label{sec:introduction}

Tabular stream learning studies supervised prediction from sequentially arriving examples under bounded memory, real-time response requirements, and potential distribution shift~\citep{aggarwal2007data,nguyen2015survey, korycki2022instance,gama2014survey}. Standard stream learners address this through online or incremental model-state adaptation, such as updating tree statistics or ensemble members
~\citep{domingos2000mining,gomes2017adaptive,gomes2019streaming}.

Recent tabular foundation models (TFMs) offer a different paradigm for tabular prediction. Given a labeled context $D$ and a query $x$, TFMs directly output a posterior predictive distribution $q_\theta(\cdot\mid x,D)$ without dataset-specific model updates~\citep{hollmann2022tabpfn,hollmann2025accurate}. This in-context mechanism makes TFMs appealing for data streams since a model can adapt by changing the retained context~\citep{lourencco2025bridging}. Therefore, the core challenge shifts from how to update the model to how to manage a context under bounded memory. A recent method,  DualFIFO~\citep{lourencco2026context}, has shown that first-in-first-out context updates can make TFMs competitive on streams, but it remains unclear which past examples should be retained for future queries.

Motivated by this gap, we introduce a \emph{future-information view} that measures how much information a context provides to near-future queries. This yields three requirements for a context management policy: preserve recent examples to approximate the near-future distribution, retain uncertain examples with high potential label information, and remove redundant examples with overlapping evidence. 

We implement these requirements as \method (\methodlong), a context update policy with entropy-gated admission and redundancy-aware eviction. Across seven streams, \method achieves the best prequential accuracy over classical stream-learning baselines, improving by up to $+19.59$ points, shows consistent gains across multiple TFM backbones, and ranks first among controlled policy variants from the same design space.
\section{Problem Setup}
\label{sec:problem_setup}

In this paper, we consider stream classification under the test-then-train prequential protocol~\citep{gama2009issues,gama2013evaluating}. At each step $t$, a
query $x_t$ arrives and the learner must predict before the true label $y_t$ is observed. After $y_t$ is revealed, the new labeled example $z_t=(x_t,y_t)$ becomes available for future predictions.

When a TFM is applied on data streams, the model outputs $q_\theta(\cdot\mid x_t, D_t)$ given a query $x_t$ and bounded context $D_t$ of previously observed feature-label pairs, while $\theta$ is fixed throughout the stream. Thus, adaptation to evolving data can occur only by changing the retained context $D_t$.

We define this task as \emph{bounded context management}. Given a pretrained TFM, a prequential stream $\{(x_t,y_t)\}_{t\ge 1}$, and a context budget $B$, the goal is to design an online update policy
$\mathcal U$ that selects the next context using only information available up to time $t$:
\begin{equation}
    D_{t+1}
    =
    \mathcal U(D_t,z_t),
    \qquad
    |D_{t+1}|\le B .
    \label{eq:context_update}
\end{equation}
As stated earlier, one notable prior method DualFIFO~\citep{lourencco2026context} shows that TFMs can be effective stream learners, but this does not provide an explicit criterion for the utility of retained examples. This motivates the future-information view introduced in Section~\ref{sec:future_utility}.


\section{A Future-Information View of Bounded Context}
\label{sec:future_utility}

A key question in bounded context management is deciding which examples should be kept and which should be removed. We formalize this using the
\emph{near-future information} provided by a context. All supporting assumptions, lemmas, and theorems are stated and proved in Appendix~\ref{app:theory}.

\paragraph{Future-Information View.}
Fix a stream step $t$ and a short horizon length $h$. Let $\mathcal H_t^+=\{t+1,\ldots,t+h\}$ and define the near-future feature distribution as
$\textstyle \mathcal P_{t,X}^+ = \frac{1}{h}\sum_{s\in\mathcal H_t^+}\mathcal P_{s,X}$,
where $\mathcal P_{s,X}$ is the feature marginal at step $s$. For a near-future query $x'\sim\mathcal P_{t,X}^+$, let $Y_{x'}$ be its label random variable.

\begin{definition}[Future usefulness of a bounded context]
\label{def:future_information}
For a current context $D_t$, we define its future
usefulness as
\begin{equation}
    \mathcal J_t(D_t)
    =
    \mathbb E_{x'\sim\mathcal P_{t,X}^+}
    \left[
        I(D_t;Y_{x'}\mid x')
    \right].
    \label{eq:future_information_objective}
\end{equation}
\end{definition}
Here, $I$ denotes pointwise mutual information evaluated at each $x'$. This says that a context is useful if it provides information about near-future query labels. However, it cannot be used as an online policy since $\mathcal P_{t,X}^+$ is unknown, and $\mathcal J_t$ is for the whole context rather than individual samples. We therefore decompose it into three practical signals.

\paragraph{Recency.}
The outer expectation in Eq.~\eqref{eq:future_information_objective} depends on the unknown near-future feature distribution.
Since recent examples often reflect the current data-generating concept~\citep{losing2017self,chen2020selective}, we use the recent window as a proxy for near-future query regions.
Assumption~\ref{assump:recent_window_stability} and
Lemma~\ref{lem:recency_proxy} further formalize this, and show that the information measured on recent examples approximates the near-future information up to a distributional discrepancy term. This motivates reserving part of the context budget for recent examples.


\paragraph{Uncertainty.}
Recency tells us where future queries may appear, but not which examples are worth keeping beyond the recent window. For a newly observed
example $z_t=(x_t,y_t)$, the ideal admission criterion is its marginal
contribution
\begin{equation}
    \Delta_t(z_t\mid D_t)
    =
    \mathcal J_t(D_t\cup\{z_t\})
    -
    \mathcal J_t(D_t).
    \label{eq:marginal_info_gain}
\end{equation}
This asks how much adding $z_t$ increases the future information of the
context. Lemma~\ref{lem:marginal_gain_identity} rewrites this as 
\begin{equation}
\Delta_t(z_t\mid D_t)
    =\mathbb E_{x'\sim\mathcal P_{t,X}^+}\left[I(Y_{x_t};Y_{x'}\mid x_t,x',D_t)\right], 
\end{equation}
meaning that $z_t$ is useful when its label helps predict future labels. 
Moreover, Theorem~\ref{thm:entropy_local_utility} shows that, under Assumptions~\ref{assump:local_entropy_consistency} and~\ref{assump:local_label_coherence}, prediction-time uncertainty $H(Y_{x_t}\mid x_t,D_t)$ provides a tractable lower-bound signal for the local future-information gain of $z_t$, up to the locality and coherence errors $\delta$ and $\epsilon$.
This motivates using predictive entropy as the admission signal for the context.

\paragraph{Redundancy.}
When the context is full, the policy must remove a stored example. The ideal policy would delete the item that causes the smallest loss in future
information:
\begin{equation}
    u^\ast
    \in
    \arg\min_{u\in D_t}
    \left[
        \mathcal J_t(D_t)-\mathcal J_t(D_t\setminus\{u\})
    \right].
    \label{eq:ideal_information_eviction}
\end{equation}
This is impractical because it requires leave-one-out TFM evaluations for all stored examples. We therefore use redundancy as a proxy for low
information loss. If two same-class examples are close in representation
space, they are likely to provide overlapping evidence for future labels, so removing one of them should lose little information. This motivates removing close same-class examples, with the formal redundancy condition given in Assumption~\ref{assump:local_label_redundancy}.
\section{CURE: Context Management for Streaming TFMs}
\label{sec:method}

\begin{figure*}[t]
    \centering
    \includegraphics[width=\textwidth]{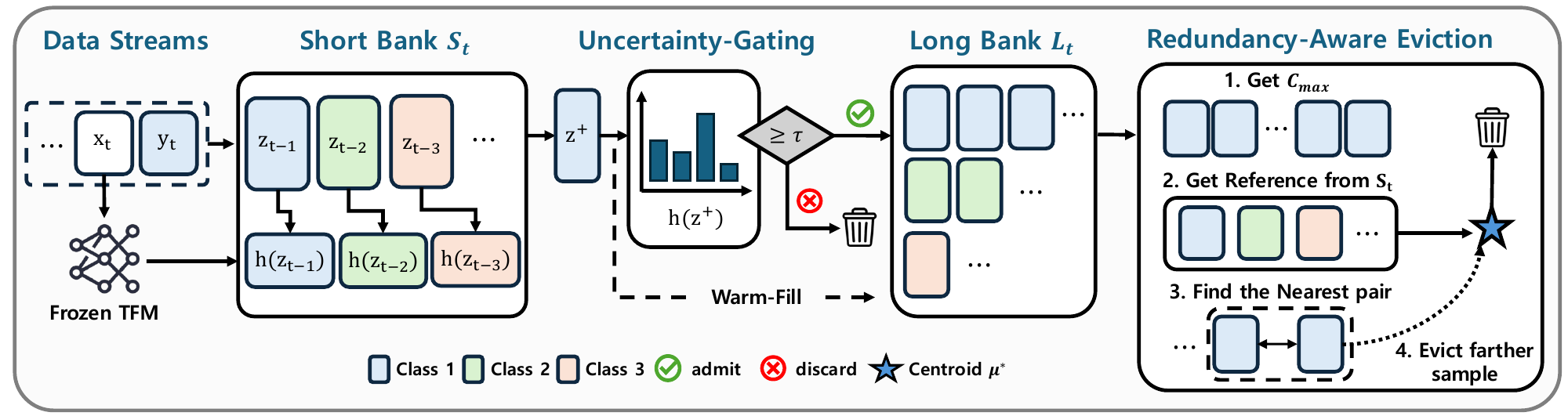}
    \caption{
    Overview of \method. A new labeled example $z_t=(x_t,y_t)$ first
    enters the short bank $S_t$ to preserve recent support. When $S_t$
    overflows, the oldest item $z^+$ becomes a long-bank candidate and is
    admitted according to its stored prediction-time entropy. When the
    long bank $L_t$ exceeds its budget, \method removes a locally
    redundant same-class example.
    }
    \label{fig:cure_architecture}
\end{figure*}

Motivated by the three signals from the future-information view, we
introduce \method (\methodlong), a context management policy for stream learning with TFMs. As depicted in Figure~\ref{fig:cure_architecture}, \method allocates part of the context budget to a short bank and the remaining budget to a long bank. New labeled examples first enter the short bank. When the short bank overflows, its oldest item becomes a candidate for long-term retention and is filtered by uncertainty-gated admission. When the long bank exceeds its budget, a redundant same-class example is removed.

\subsection{Dual-Bank Context}

The recency signal suggests that a context should preserve recent examples as a proxy for near-future query regions. To implement this, \method maintains
\begin{equation}
    D_t=S_t\cup L_t,
    \;
    |S_t|\le B_S,\; |L_t|\le B_L,\; B_S+B_L=B,
\end{equation}
where $S_t$ is a FIFO short bank and $L_t$ is a long bank. The short bank
contains the most recent labeled examples as a sliding window, while the
long bank stores older examples selected by uncertainty-gated admission
and redundancy-aware eviction. Thus, every example remains available for
at least the short-bank duration.

\subsection{Entropy-Gated Long-Bank Admission}

The uncertainty signal suggests that high-entropy examples can provide useful label information once revealed. We apply this signal to the long bank.

When an example $z_t=(x_t,y_t)$ is first predicted, the TFM outputs
$p_t=q_\theta(\cdot\mid x_t,D_t)$ before observing $y_t$. We attach to
$z_t$ its normalized predictive entropy
\begin{equation}
    h(z_t)
    =
    \frac{-\sum_{c=1}^{C}p_t(c)\log p_t(c)}{\log C},
    \label{eq:stored_entropy}
\end{equation}
where $C$ is the number of classes. This score is computed once at
prediction time and reused when the example later leaves the short bank.

Let $z^+=(x^+,y^+)$ be the item that overflows from $S_t$. During
\emph{warm-fill}, candidates are inserted into $L_t$ until the long bank is full. Afterward, a candidate is admitted only if its stored entropy exceeds a
threshold $\tau$:
\begin{equation}
    \mathrm{Admit}(z^+)
    =
    \begin{cases}
    1, & |L_t|<B_L,\\
    \mathbb I\{h(z^+)\ge \tau\}, & |L_t|\ge B_L .
    \end{cases}
    \label{eq:generic_admission_rule}
\end{equation}
Thus, the long bank retains examples that were uncertain
when first observed and are more likely to help future queries.

\subsection{Redundancy-Aware Long-Bank Eviction}

The redundancy signal suggests that close same-class examples provide
overlapping evidence, so one can be removed with limited information loss. We again apply this signal only inside the long bank.

When $L_t$ exceeds its budget, \method first selects the most represented
class $c_{\max}
    \in
    \arg\max_c |\{z\in L_t:y_z=c\}|$,
which reduces the risk of deleting sparse class evidence. Within this
class, \method finds the closest same-class pair in the normalized
raw-feature representation $\phi(x)$:
\begin{equation}
    (i^\ast,j^\ast)
    \in
    \arg\min_{\substack{i\ne j\\ z_i,z_j\in L_t\\ y_i=y_j=c_{\max}}}
    \|\phi(x_i)-\phi(x_j)\|_2 .
    \label{eq:same_class_pair}
\end{equation}
This pair is treated as the most redundant local evidence in the overrepresented class. To choose which endpoint to remove from $x_{i^*}$ and $x_{j^*}$, \method uses the recent short bank as a reference for the current regime. It computes a recent centroid for class $c_{\max}$ and removes the endpoint farther from it. Full procedural details are provided in Appendix~\ref{app:algorithm_details}.
\section{Experiments}
\label{sec:experiments}

\begin{table*}[t]
\centering
\caption{Prequential accuracy on data streams. Values in parentheses next to \method indicate absolute gains over the best classical baseline and $\pm$ denotes standard deviation over five seeds. Best and second-best results are bolded and underlined.}
\label{tab:main_results}
\vspace{0.25em}
\resizebox{0.88\textwidth}{!}{
\footnotesize
\setlength{\tabcolsep}{3.5pt}
\renewcommand{\arraystretch}{1.05}
\begin{tabular}{@{}llcccccc@{}}
\toprule
Dataset & \method{} & ARF & BOLE & LevBag & SRP & EFDT & VFDT \\
\midrule
NOAA    
& $\mathbf{81.94}_{\scriptscriptstyle(+2.51)}$ 
& $79.42{\pm}0.08$ 
& $74.97{\pm}0.19$ 
& $79.40{\pm}0.10$ 
& $\underline{79.43{\pm}0.10}$ 
& $75.56$ 
& $75.66$ \\

METER   
& $\mathbf{90.80}_{\scriptscriptstyle(+19.31)}$ 
& $68.27{\pm}0.09$ 
& $68.06{\pm}0.27$ 
& $69.40{\pm}0.25$ 
& $\underline{71.49{\pm}0.18}$ 
& $61.39$ 
& $55.44$ \\

RIALTO  
& $\mathbf{92.04}_{\scriptscriptstyle(+19.59)}$ 
& $63.40{\pm}0.12$ 
& $53.62{\pm}0.32$ 
& $64.13{\pm}0.18$ 
& $\underline{72.45{\pm}0.15}$ 
& $58.73$ 
& $40.46$ \\

POSTURE-No8
& $\mathbf{62.10}_{\scriptscriptstyle(+1.58)}$ 
& $59.44{\pm}0.08$ 
& $54.60{\pm}0.38$ 
& $\underline{60.52{\pm}0.04}$ 
& $58.88{\pm}0.15$ 
& $55.34$ 
& $53.11$ \\

POKER   
& $\mathbf{99.60}_{\scriptscriptstyle(+2.07)}$ 
& $84.05{\pm}0.11$ 
& $88.46{\pm}0.19$ 
& $\underline{97.53{\pm}0.10}$ 
& $89.29{\pm}0.16$ 
& $82.75$ 
& $89.09$ \\

NOMAO   
& $\mathbf{97.87}_{\scriptscriptstyle(+0.45)}$ 
& $97.19{\pm}0.04$ 
& $95.82{\pm}0.12$ 
& $97.27{\pm}0.03$ 
& $\underline{97.42{\pm}0.04}$ 
& $94.55$ 
& $93.98$ \\

AGR(A)  
& $\mathbf{90.93}_{\scriptscriptstyle(+0.67)}$ 
& $\underline{90.26{\pm}0.08}$ 
& $88.29{\pm}0.10$ 
& $87.62{\pm}0.29$ 
& $90.13{\pm}0.15$ 
& $75.15$ 
& $76.32$ \\

\midrule
Avg. rank    
& $\mathbf{1.00}$ 
& $3.71$ 
& $5.43$ 
& $3.14$ 
& $\underline{2.57}$ 
& $6.00$ 
& $6.14$ \\
\bottomrule
\end{tabular}
}
\vspace{-0.5em}
\end{table*}
\begin{table}[t]
\centering
\caption{Backbone robustness across data streams. Values compare \method{} with DualFIFO using the same TFM backbone. \method{} improves in 17 of 18 comparisons.}
\label{tab:robustness}
\vspace{0.25em}
\resizebox{0.49\textwidth}{!}{
\footnotesize
\setlength{\tabcolsep}{3.5pt}
\renewcommand{\arraystretch}{1.05}
\begin{tabular}{@{}lcccccc@{}}
\toprule
& \multicolumn{2}{c}{LimiX-v1}
& \multicolumn{2}{c}{TabPFN-v2.5}
& \multicolumn{2}{c}{TabDPT-v1} \\
\cmidrule(lr){2-3}
\cmidrule(lr){4-5}
\cmidrule(lr){6-7}
Dataset
& \method{} & DualFIFO
& \method{} & DualFIFO
& \method{} & DualFIFO \\
\midrule
NOAA   
& $\mathbf{81.43}$ & $80.83$
& $\mathbf{81.82}$ & $81.26$
& $\mathbf{81.94}$ & $81.49$ \\

METER  
& $\mathbf{89.06}$ & $88.69$
& $\mathbf{87.88}$ & $87.64$
& $\mathbf{80.73}$ & $79.65$ \\

RIALTO 
& $\mathbf{90.17}$ & $89.59$
& $\mathbf{89.53}$ & $88.93$
& $\mathbf{88.05}$ & $87.64$ \\

POSTURE-No8
& $\mathbf{60.50}$ & $60.47$
& $60.83$ & $\mathbf{60.99}$
& $\mathbf{61.67}$ & $61.50$ \\

NOMAO
& $\mathbf{97.83}$ & $97.54$
& $\mathbf{97.85}$ & $97.55$
& $\mathbf{97.34}$ & $96.58$ \\

AGR(A)
& $\mathbf{90.56}$ & $90.54$
& $\mathbf{91.02}$ & $90.85$
& $\mathbf{90.02}$ & $89.93$ \\
\bottomrule
\end{tabular}
}
\vspace{-0.5em}
\end{table}

To evaluate whether \method provides an effective context management protocol for streaming TFMs, we ask two questions. First, can bounded context management make TFMs competitive with classical stream learning methods? Second, does the policy transfer across different TFM backbones?

\subsection{Experimental Setup}

\paragraph{Datasets.}
We evaluate on seven streams: five real-world streams from the USP Data
Stream repository\footnote{\url{https://sites.google.com/view/uspdsrepository}}
(NOAA, METER, RIALTO, POSTURE-No8, POKER), one additional real-world
stream (NOMAO), and one synthetic stream AGR(A)~\citep{agrawal1993mining}.
POSTURE-No8 is a 10-class variant of POSTURE after removing the
rarest class. Dataset statistics are provided in
Appendix~\ref{app:datasets}.

\paragraph{Models and baselines.}
We use TabICL-v2~\citep{qu2026tabiclv2} as the primary backbone. We compare \method against representative stream-learning baselines implemented in MOA 24.07~\citep{bifet2010moa}: VFDT, EFDT, ARF, SRP, LevBag, and BOLE~\citep{domingos2000mining,
manapragada2018extremely,gomes2017adaptive,gomes2019streaming,
bifet2010leveraging,de2016boosting}. To test transferability across TFMs, we also evaluate LimiX-v1, TabPFN-v2.5, and TabDPT-v1~\citep{zhang2025limix,grinsztajn2025tabpfn,ma2024tabdpt}.
Model and baseline details are given in Appendix~\ref{app:backbones} and~\ref{app:baselines}.

\paragraph{Evaluation protocol.}
TFM policies use a total context budget of $B=1000$, a short-bank ratio of $\rho=0.75$, and a warm-up period of 100 stream steps. We update the
context after every arriving labeled example and report cumulative
prequential accuracy. Implementation details including all hyperparameters are provided in
Appendix~\ref{app:evaluation_details}.

\subsection{Main Results}

Table~\ref{tab:main_results} compares \method with TabICL-v2 against
classical stream-learning baselines. \method achieves the best
prequential accuracy on all seven streams and obtains the best average
rank. Its gains range from $+0.45$ points on NOMAO to $+19.31$ on METER and $+19.59$ on RIALTO. This corresponds to up to $27.0\%$ improvement and $9.0\%$ average
relative improvement across datasets.

These results show that a TFM with bounded context management can be a strong stream learner compared with online tree and ensemble methods. This is notable because \method adapts to evolving streams only by updating the retained context. 

\subsection{Backbone Transferability}

Table~\ref{tab:robustness} evaluates whether the gains of \method transfer across multiple TFM backbones. This matters because bounded context management operates at the input-context level and should not rely on a backbone-specific architecture.

Across LimiX-v1, TabPFN-v2.5, and TabDPT-v1, \method improves over the prior method DualFIFO~\citep{lourencco2026context} in 17 of 18 comparisons. The gains are modest on AGR(A) and POSTURE-No8, but larger on NOAA, METER, RIALTO, and NOMAO. This indicates that the benefit is not specific to one backbone, and supports that \method is a transferable interface for stream learning with TFMs.

\subsection{Design-Space Ablation}
\label{sec:controlled_ablation}

We further compare \method with controlled policy variants from the same design space that remove or alter uncertainty-gated admission and redundancy-aware eviction. \method achieves the best average rank, suggesting that the proposed signals from the future-information view are complementary. Full definitions and results are provided in Appendix~\ref{app:policy_variants}.
\section{Conclusion}
\label{sec:conclusion}

This work studies stream learning with TFMs from the perspective of bounded context management. Unlike classical stream learners that adapt by updating model states, a TFM adapts through the labeled examples retained as context. We formalize this through a future-information view, which connects context usefulness to the
information it provides about near-future queries. This leads to \method, a simple policy that combines recent support, entropy-gated admission, and redundancy-aware eviction. Across multiple streams and TFM backbones, our results suggest that context management is a central mechanism for making TFMs effective on evolving data streams.

\newpage
\bibliography{reference}
\bibliographystyle{icml2026}

\newpage
\appendix
\onecolumn
\section{Theory Details for the Future-Information View}
\label{app:theory}

This appendix provides the formal assumptions and proofs supporting Section~\ref{sec:future_utility}. We fix a stream step $t$, a current context $D_t$, and a newly observed candidate $z_t=(x_t,y_t)$. The
near-future feature distribution is
\[
    \mathcal P_{t,X}^+
    =
    \frac{1}{h}
    \sum_{s\in\mathcal H_t^+}
    \mathcal P_{s,X},
    \qquad
    \mathcal H_t^+=\{t+1,\ldots,t+h\}.
\]
For a future feature value $x'$, let $Y_{x'}$ denote its label random variable. For the current candidate feature $x_t$, we use $Y_{x_t}$ to denote the label random variable before its true label is observed.

\subsection{Recent-Window Approximation}
\label{app:recency_theory}

The recency signal is motivated by replacing the unavailable near-future feature distribution with the empirical distribution of recent examples.

Let $S_t$ be the recent window of labeled examples before time $t$, and let $\widehat{\mathcal P}_{S_t}$ be its empirical feature distribution. Replacing the unknown near-future feature distribution with the recent window gives the recent-window objective
\begin{equation}
    \widehat{\mathcal J}_{S_t}(D_t)
    =
    \mathbb E_{x_t\sim\widehat{\mathcal P}_{S_t}}
    \left[
        I(D_t;Y_{x_t}\mid x_t)
    \right].
    \label{eq:recent_information_objective}
\end{equation}

The following assumption states the condition under which the recent
window can act as a proxy for the near future: the two feature
distributions should not be too far apart.

\begin{assumption}[Recent-window local stability]
\label{assump:recent_window_stability}
At time $t$, the empirical feature distribution of the recent window is
close to the near-future feature distribution:
\begin{equation}
    W_1
    \left(
        \widehat{\mathcal P}_{S_t},
        \mathcal P_{t,X}^+
    \right)
    \le
    \rho_{S_t},
    \label{eq:local_stationarity_recent}
\end{equation}
where $W_1$ is the Wasserstein-1 distance in the representation space and
$\rho_{S_t}$ is the discrepancy between the recent-window and near-future
feature distributions at time $t$.
\end{assumption}

Under this condition, the following lemma shows that evaluating future
information on the recent window is close to evaluating it on the ideal
near-future distribution.

\begin{lemma}[Recent-window approximation]
\label{lem:recency_proxy}
For a fixed context $D_t$, suppose the query-wise information function
$x\mapsto I(D_t;Y_x\mid x)$ is $L$-Lipschitz in the representation space.
Under Assumption~\ref{assump:recent_window_stability},
\begin{equation}
    \left|
        \mathcal J_t(D_t)
        -
        \widehat{\mathcal J}_{S_t}(D_t)
    \right|
    \le
    L\rho_{S_t}.
    \label{eq:recency_proxy_bound}
\end{equation}
\end{lemma}

This bound is the formal reason for reserving a short-bank budget: if the
recent window is close to the near future, the information measured on
recent examples is a reliable proxy for near-future information.

\begin{proof}
Define the query-wise information function for the current context as
\[
    f_{D_t}(x)=I(D_t;Y_x\mid x).
\]
By assumption, $f_{D_t}$ is $L$-Lipschitz in the representation space. The
near-future objective and the recent-window objective can be written as
\[
    \mathcal J_t(D_t)
    =
    \mathbb E_{x\sim\mathcal P_{t,X}^+}[f_{D_t}(x)],
    \qquad
    \widehat{\mathcal J}_{S_t}(D_t)
    =
    \mathbb E_{x\sim\widehat{\mathcal P}_{S_t}}[f_{D_t}(x)].
\]
By the Kantorovich--Rubinstein duality for Wasserstein-1 distance,
\[
\left|
    \mathbb E_{x\sim P}[f_{D_t}(x)]
    -
    \mathbb E_{x\sim Q}[f_{D_t}(x)]
\right|
\le
L W_1(P,Q)
\]
for any two feature distributions $P$ and $Q$. Taking
$P=\mathcal P_{t,X}^+$ and $Q=\widehat{\mathcal P}_{S_t}$ gives
\[
    \left|
        \mathcal J_t(D_t)-\widehat{\mathcal J}_{S_t}(D_t)
    \right|
    \le
    L W_1(\widehat{\mathcal P}_{S_t},\mathcal P_{t,X}^+).
\]
Applying Assumption~\ref{assump:recent_window_stability} yields
Eq.~\eqref{eq:recency_proxy_bound}.
\end{proof}

\subsection{Marginal Future Information}
\label{app:marginal_theory}

The context-level objective $\mathcal J_t(D_t)$ scores an entire context,
but the online policy must decide whether a single new example should be
kept. The next lemma connects this instance-level decision to future-label
information.

For the newly observed candidate $z_t=(x_t,y_t)$, define its marginal
future information by
\[
    \Delta_t(z_t\mid D_t)
    =
    \mathcal J_t(D_t\cup\{z_t\})
    -
    \mathcal J_t(D_t).
\]

\begin{lemma}[Item contribution as future-label information]
\label{lem:marginal_gain_identity}
For the candidate $z_t=(x_t,y_t)$,
\begin{equation}
    \Delta_t(z_t\mid D_t)
    =
    \mathbb E_{x'\sim\mathcal P_{t,X}^+}
    \left[
        I(Y_{x_t};Y_{x'}\mid x_t,x',D_t)
    \right].
    \label{eq:marginal_gain_identity}
\end{equation}
\end{lemma}

This identity says that a candidate is useful when its label carries
information about labels of near-future queries. It is the bridge from the
context-level objective to an admission rule for individual examples.

\begin{proof}
By definition,
\begin{align*}
    \Delta_t(z_t\mid D_t)
    &=
    \mathcal J_t(D_t\cup\{z_t\})-\mathcal J_t(D_t) \\
    &=
    \mathbb E_{x'\sim\mathcal P_{t,X}^+}
    \left[
        I(D_t\cup\{z_t\};Y_{x'}\mid x')
        -
        I(D_t;Y_{x'}\mid x')
    \right].
\end{align*}
The candidate $z_t$ is represented by its feature value $x_t$ and label
random variable $Y_{x_t}$. By the chain rule for mutual information,
\[
    I(D_t\cup\{z_t\};Y_{x'}\mid x')
    =
    I(D_t;Y_{x'}\mid x')
    +
    I(Y_{x_t};Y_{x'}\mid x_t,x',D_t).
\]
Substituting this identity into the previous expression gives
Eq.~\eqref{eq:marginal_gain_identity}.
\end{proof}

\subsection{Entropy Lower Bound for Local Information}
\label{app:entropy_theory}

Lemma~\ref{lem:marginal_gain_identity} gives an ideal admission criterion, but it still depends on unknown future labels. We therefore relate this quantity to a tractable prediction-time signal: the entropy of the candidate before its label is observed.

Let $\mathcal N_t(z_t)\subseteq\mathcal X$ denote a measurable local region around $x_t$. We call it the effective future region of $z_t$, where the revealed label of $z_t$ can provide local evidence for near-future queries. Define its near-future mass as
\begin{equation}
\alpha_t(z_t)
=
\Pr_{x’\sim \mathcal P_{t,X}^+}
\left[
x’\in \mathcal N_t(z_t)
\right].
\label{eq:future_region_mass}
\end{equation}

We first require the candidate’s prediction-time uncertainty to be representative of the uncertainty in its local future region.

\begin{assumption}[Local entropy consistency]
\label{assump:local_entropy_consistency}
For the newly observed candidate $z_t=(x_t,y_t)$,
\begin{equation}
\mathbb E
\left[
H(Y_{x’}\mid x_t,x’,D_t)
\mid
x’\in \mathcal N_t(z_t)
\right]
\ge
H(Y_{x_t}\mid x_t,D_t)-\delta .
\label{eq:local_entropy_consistency}
\end{equation}
\end{assumption}

This condition says that the candidate should not be much more uncertain than the nearby future queries it is intended to support. The slack variable $\delta$ captures violations of this local entropy consistency.

We also require the revealed label of the candidate to be informative about labels in its local future region.

\begin{assumption}[Local label coherence]
\label{assump:local_label_coherence}
For the newly observed candidate $z_t=(x_t,y_t)$,
\begin{equation}
\mathbb E
\left[
H(Y_{x’}\mid Y_{x_t},x_t,x’,D_t)
\mid
x’\in \mathcal N_t(z_t)
\right]
\le
\epsilon .
\label{eq:local_label_coherence}
\end{equation}
\end{assumption}

This assumption does not require labels to be deterministic functions of the input. Rather, the residual term $\epsilon$ captures stochastic labels, unobserved variables, and class overlap within the local region. When nearby labels are noisy or only weakly coherent, $\epsilon$ becomes larger and the resulting lower bound becomes weaker.

Together, these two local conditions connect prediction-time entropy to the future-label information supplied by $z_t$.

\begin{theorem}[Entropy lower bound for local information]
\label{thm:entropy_local_utility}
For a candidate $z_t=(x_t,y_t)$ and its effective future region
$\mathcal N_t(z_t)$, the expected local information gain satisfies
\begin{align}
&\mathbb E_{x'\sim\mathcal P_{t,X}^+}
\left[
    \mathbf 1\{x'\in\mathcal N_t(z_t)\}
    I(Y_{x_t};Y_{x'}\mid x_t,x',D_t)
\right] \nonumber\\
&\qquad\ge
\alpha_t(z_t)
\left[
    H(Y_{x_t}\mid x_t,D_t)-\delta-\epsilon
\right].
\label{eq:entropy_local_bound}
\end{align}
\end{theorem}

Theorem~\ref{thm:entropy_local_utility} provides the formal motivation for entropy-gated admission. It shows that, when local entropy consistency and local label coherence approximately hold, higher prediction-time entropy yields a larger lower-bound signal for the local future-information gain of $z_t$.

\begin{proof}
For any $x'\in\mathcal N_t(z_t)$, the definition of conditional mutual
information gives
\[
    I(Y_{x_t};Y_{x'}\mid x_t,x',D_t)
    =
    H(Y_{x'}\mid x_t,x',D_t)
    -
    H(Y_{x'}\mid Y_{x_t},x_t,x',D_t).
\]
Taking expectation over near-future feature values and restricting to the
effective future region,
\begin{align*}
&\mathbb E_{x'\sim\mathcal P_{t,X}^+}
\left[
    \mathbf 1\{x'\in\mathcal N_t(z_t)\}
    I(Y_{x_t};Y_{x'}\mid x_t,x',D_t)
\right] \\
&=
\alpha_t(z_t)
\mathbb E
\left[
    I(Y_{x_t};Y_{x'}\mid x_t,x',D_t)
    \mid x'\in\mathcal N_t(z_t)
\right] \\
&=
\alpha_t(z_t)
\mathbb E
\left[
    H(Y_{x'}\mid x_t,x',D_t)
    -
    H(Y_{x'}\mid Y_{x_t},x_t,x',D_t)
    \mid x'\in\mathcal N_t(z_t)
\right].
\end{align*}
By Assumption~\ref{assump:local_entropy_consistency},
\[
\mathbb E
\left[
    H(Y_{x'}\mid x_t,x',D_t)
    \mid x'\in\mathcal N_t(z_t)
\right]
\ge
H(Y_{x_t}\mid x_t,D_t)-\delta .
\]
By Assumption~\ref{assump:local_label_coherence},
\[
\mathbb E
\left[
    H(Y_{x'}\mid Y_{x_t},x_t,x',D_t)
    \mid x'\in\mathcal N_t(z_t)
\right]
\le
\epsilon .
\]
Combining these two inequalities proves Eq.~\eqref{eq:entropy_local_bound}.
\end{proof}

\subsection{Same-Class Evidence Redundancy}
\label{app:redundancy_theory}

The removal problem asks which stored example can be deleted with minimal
loss in future information. The following assumption formalizes the
intuition that close same-class examples provide overlapping evidence.

\begin{assumption}[Same-class evidence redundancy]
\label{assump:local_label_redundancy}
For two same-class context items $z_i=(x_i,y_i)$ and $z_j=(x_j,y_j)$ in
$D_t$ with $y_i=y_j$, there exists a nondecreasing function $\eta$ with
$\eta(r)\to0$ as $r\to0$ such that
\begin{equation}
    \mathcal J_t(D_t)-\mathcal J_t(D_t\setminus\{z_j\})
    \le
    \eta(\|\phi(x_i)-\phi(x_j)\|_2).
    \label{eq:redundancy_info_bound}
\end{equation}
\end{assumption}

This assumption says that if a stored example has a close same-class
neighbor in representation space, then removing it causes only a small
loss of future information. This is the formal motivation for removing
close same-class examples from the long bank when the budget is full.
\section{Additional Analyses}
\label{app:additional_analyses}

\subsection{Controlled Policy Variants}
\label{app:policy_variants}

We first evaluate controlled context-policy variants to isolate the design choices of \method. The variants summarized in Table~\ref{tab:policy_axes} test whether uncertainty is needed for retaining informative candidates, whether redundancy is needed for
low-loss removal, and whether entropy and same-class nearest-neighbor removal are the right instantiations of these signals. All variants use
the same TFM backbone and the same total context budget. We exclude POKER from this design-space evaluation due to its large stream length and
substantially longer runtime.

\begin{table}[h]
\centering
\caption{TFM context-policy variants for the design-space evaluation. They differ only in how they instantiate uncertainty-based retention and redundancy-aware removal.}
\label{tab:policy_axes}
\vspace{0.25em}
\setlength{\tabcolsep}{3pt}
\begin{tabular}{lcccc}
\toprule
Policy & Recency & Long bank & Retention signal & Removal rule \\
\midrule
\method{} & \checkmark & \checkmark & entropy & same-class NN \\
DualFIFO & \checkmark & \checkmark & -- & class-aware FIFO \\
Entropy-only & \checkmark & \checkmark & entropy & class-aware FIFO \\
Redundancy-only & \checkmark & \checkmark & -- & same-class NN \\
CURE-Margin & \checkmark & \checkmark & margin & same-class NN \\
CURE-GlobalNN & \checkmark & \checkmark & entropy & global NN \\
\bottomrule
\end{tabular}
\vspace{-0.6em}
\end{table}
\begin{table}[h]
\centering
\caption{Design-space ablation with prequential accuracy. Best results are bolded and second-best results are underlined. \method achieves the best average rank, showing that the theory-driven signals are more effective when used together than in isolation.}
\label{tab:ablation_variants}
\vspace{0.25em}
\small
\setlength{\tabcolsep}{4.5pt}
\renewcommand{\arraystretch}{0.9}
\begin{tabular}{lcccccc}
\toprule
Dataset
& \method{}
& DualFIFO
& Entropy-only
& Redundancy-only
& CURE-Margin
& CURE-GlobalNN \\
\midrule
NOAA
& $\underline{81.94}$
& 81.42
& 81.78
& 81.64
& 81.90
& $\mathbf{81.95}$ \\

METER
& $\mathbf{90.80}$
& 90.60
& 89.94
& 90.63
& 90.53
& $\underline{90.68}$ \\

RIALTO
& $\mathbf{92.04}$
& 91.72
& 91.36
& 91.97
& 92.00
& $\underline{92.02}$ \\

POSTURE-No8
& $\mathbf{62.10}$
& 61.85
& 61.97
& $\underline{62.01}$
& $\underline{62.01}$
& 61.88 \\

NOMAO
& $\mathbf{97.87}$
& 97.47
& 97.65
& 97.64
& $\underline{97.81}$
& 97.75 \\

AGR(A)
& $\mathbf{90.93}$
& 90.45
& $\underline{90.84}$
& 90.54
& 90.74
& 90.62 \\

\midrule
Avg. rank
& $\mathbf{1.17}$
& 5.50
& 4.33
& 4.08
& 3.08
& $\underline{2.83}$ \\
\bottomrule
\end{tabular}
\vspace{-0.5em}
\end{table}

\textbf{DualFIFO} is the FIFO-style dual-memory reference policy. It uses
the same short-bank and long-bank budgets as \method. New examples enter
the short bank, and short-bank overflow candidates are inserted into the
long bank. When the long bank exceeds its budget, DualFIFO selects the
most represented class and removes the oldest long-bank item from that
class. Thus, DualFIFO preserves recency and coarse class balance, but does
not use uncertainty or redundancy-aware removal.

\textbf{Entropy-only} isolates uncertainty-based retention. It uses the
same entropy-gated long-bank admission rule as \method, but keeps the
DualFIFO class-aware FIFO removal rule when the long bank overflows. This
tests whether uncertainty-based candidate filtering alone is sufficient
without redundancy-aware capacity management.

\textbf{Redundancy-only} isolates redundancy-aware removal. It admits
every short-bank overflow candidate into the long bank, but uses the
same-class nearest-neighbor removal rule from \method when the long bank
overflows. This tests whether removing duplicated same-class evidence is
sufficient without uncertainty-based candidate filtering.

\textbf{CURE-Margin} tests the choice of uncertainty score. It keeps the
same dual-memory structure and same-class redundancy-aware removal rule as
\method, but replaces predictive entropy with a top-two margin score.
For a prediction distribution $p_t=q_\theta(\cdot\mid x_t,D_t)$, let
$p_{t,(1)}$ and $p_{t,(2)}$ be the largest and second-largest class
probabilities. The margin-based uncertainty score is
\[
    h_{\mathrm{margin}}(z_t)
    =
    1-\left(p_{t,(1)}-p_{t,(2)}\right).
\]
A larger value indicates a smaller gap between the two most likely
classes.

\textbf{CURE-GlobalNN} tests the label-conditional nature of redundancy. It keeps entropy-gated admission but replaces same-class nearest-neighbor removal with class-agnostic nearest-neighbor removal. When the long bank overflows, it finds
\begin{equation*}
    (i^\ast,j^\ast)
    \in
    \arg\min_{\substack{i\ne j\\ z_i,z_j\in L_t}}
    \|\phi(x_i)-\phi(x_j)\|_2 .
    \label{eq:global_pair}
\end{equation*}
It then removes one endpoint using the same centroid-based tie-breaking rule as \method. This variant tests whether redundancy should be defined purely geometrically or conditioned on the label.

Table~\ref{tab:ablation_variants} compares \method with these controlled
policy variants on six streams. \method improves over DualFIFO on all
six streams, achieves the best result on five, and has the best average
rank. Because DualFIFO already uses the same dual-memory structure with
recency and class-coverage bias, these gains indicate the efficacy of the
future-information-guided design beyond FIFO retention.

The component variants show that uncertainty and redundancy are
complementary. Entropy-only helps on some streams, but can be unstable
when removal remains age-based. Redundancy-only consistently improves over
DualFIFO, but does not reach \method. The design-choice variants further
support our specific instantiation since predictive entropy generally performs
better than the top-two margin score, and same-class nearest-neighbor
removal gives a better average rank than class-agnostic removal.

\begin{table}[h]
\centering
\caption{Sensitivity of CURE to the entropy threshold $\tau$ with TabICLv2. Values are prequential accuracies.}
\label{tab:tau_sensitivity}
\small
\setlength{\tabcolsep}{4.5pt}
\begin{tabular}{lcccccc}
\toprule
Dataset & $\tau=0$ & $\tau=0.1$ & $\tau=0.2$ & $\tau=0.3$ & $\tau=0.4$ & $\tau=0.5$ \\
\midrule
NOAA        & 81.64 & 81.64 & 81.83 & 81.82 & \textbf{81.94} & 81.88 \\
METER       & 90.63 & \textbf{90.80} & 90.62 & 90.68 & 90.62 & 90.48 \\
RIALTO      & 91.97 & \textbf{92.04} & 91.98 & 92.00 & 92.02 & 91.96 \\
POSTURE-No8 & 62.01 & 62.03 & \textbf{62.10} & 62.03 & 62.03 & 62.08 \\
NOMAO & 97.64 & 97.71 & 97.82 & \textbf{97.87} & 97.86 & 97.86 \\
AGR(A) & 90.54 & 90.49 & 90.61 & 90.52 & \textbf{90.93} & 90.61 \\
\bottomrule
\end{tabular}
\end{table}

\subsection{Sensitivity to the Entropy Threshold $\tau$}
\label{app:tau_sensitivity}

Table~\ref{tab:tau_sensitivity} studies the sensitivity of \method to
the entropy threshold $\tau$. The $\tau=0$ column disables
uncertainty-gated admission and therefore corresponds to the
Redundancy-only variant. Moving from $\tau=0$ to a positive threshold
improves the best accuracy on all streams, indicating that the uncertainty
gate contributes beyond redundancy-aware eviction alone.

At the same time, \method is not sensitive to a narrowly tuned
threshold. Across positive thresholds
$\tau\in\{0.1,0.2,0.3,0.4,0.5\}$, the largest within-dataset accuracy
difference is only $0.39$ percentage points. Thus, uncertainty-gated
admission is useful, while the method remains stable over a broad range of
positive thresholds.

\begin{table}[h]
\centering
\caption{Average per-step runtime over seven streams. Values are in seconds.}
\label{tab:runtime_summary}
\small
\begin{tabular}{lcccc}
\toprule
Method & Total & Fit & Predict & Evict \\
\midrule
CURE     & 0.0283 & 0.0044 & 0.0226 & 0.0013 \\
DualFIFO & 0.0259 & 0.0040 & 0.0214 & 0.0005 \\
\bottomrule
\end{tabular}
\end{table}

\subsection{Efficiency Analysis}
\label{app:efficiency}

Table~\ref{tab:runtime_summary} reports the average per-step runtime of
\method and DualFIFO, decomposed into context fitting, prediction, and
eviction. The reported ``Total'' column sums the measured core components
used by the streaming prediction loop.

\method introduces additional computation over DualFIFO because it
performs entropy-gated long-bank admission and redundancy-aware eviction,
whereas DualFIFO updates the long bank using FIFO-style rules.
Nevertheless, the absolute overhead is small. The average total step time
increases from $0.0259$ seconds to $0.0283$ seconds, an increase of about
$0.0024$ seconds per example. The eviction component itself takes only
$0.0013$ seconds per step on average for \method, while prediction takes
$0.0226$ seconds. Thus, most of the runtime remains dominated by the
frozen TFM context-fitting and prediction calls, not by the
redundancy-aware removal logic. This suggests that the additional policy
logic in \method is practical for the evaluated streaming setting.
\section{Related Work}
\label{app:related_work}

\paragraph{Classical stream learning.}
Supervised stream learning is commonly studied under the test-then-train prequential protocol, where each example must be predicted before its
label is observed. Classical baselines address this setting by updating the learner state.
Hoeffding-tree methods such as VFDT and EFDT maintain node-level sufficient statistics and use statistical tests to decide when to grow or revise tree splits~\citep{domingos2000mining,manapragada2018extremely}.
Ensemble methods such as Leveraging Bagging, BOLE, Adaptive Random Forests, and Streaming Random Patches maintain multiple online learners
and adapt them through resampling, randomization, and drift-aware replacement mechanisms~\citep{bifet2010leveraging,de2016boosting, gomes2017adaptive,gomes2019streaming}. Recent stream-learning work further studies dynamic ensemble diversification, online concept-drift detection, and neural stream classifiers~\citep{abadifard2023dyned,wan2024online,su2024elastic}. These approaches adapt by changing model states, whereas our work studies adaptation through the retained context of a TFM.

\paragraph{Tabular foundation models.}
Tabular foundation models are in-context predictors that adapt to new datasets by conditioning on labeled examples, rather than by performing dataset-specific training. This paradigm follows prior-data fitted networks, where a transformer is pretrained to approximate Bayesian
posterior predictive inference from context examples~\citep{
muller2021transformers}. TabPFN introduced this idea for tabular classification using synthetic prior-generated tables~\citep{hollmann2022tabpfn}, and recent TFMs such as TabICL, LimiX, and TabDPT extend the paradigm with improved architectures, training procedures, or pretraining data~\citep{qu2025tabicl,zhang2025limix,ma2024tabdpt}. Because TFMs expose the labeled context as their adaptation interface, they are naturally suited to bounded-memory stream learning.

\section{Algorithmic Details}
\label{app:algorithm_details}

This appendix gives the procedural form of the context update used by \method. The short bank $S$ implements the recency signal, while the long bank $L$ stores older examples selected by entropy-gated admission and redundancy-aware eviction. The stored score $h(z)$ is the prediction-time entropy of $z$, computed when $z$ was first observed and later used when the sample overflows from the short bank.

\begin{algorithm}[h]
\caption{\method memory update}
\label{alg:harmony_update}
\begin{algorithmic}[1]
\REQUIRE Short bank $S$, long bank $L$, new labeled sample $z=(x,y)$, entropy $h(z)$, capacities $B_S,B_L$, threshold $\tau$
\ENSURE Updated short and long banks $S,L$
\STATE $S \leftarrow S \cup \{z\}$
\IF{$|S| > B_S$}
    \STATE $z^+=(x^+,y^+) \leftarrow$ oldest sample in $S$
    \STATE $S \leftarrow S \setminus \{z^+\}$
    \IF{$|L| < B_L$ \OR $h(z^+) \ge \tau$}
        \STATE $L \leftarrow L \cup \{z^+\}$
    \ENDIF
\ENDIF
\IF{$|L| > B_L$}
    \STATE $c_{\max} \leftarrow \arg\max_c |\{(x_i,y_i)\in L:y_i=c\}|$
    \STATE $(i^*,j^*) \leftarrow
    \arg\min_{\substack{i\ne j\\ z_i,z_j\in L\\ y_i=y_j=c_{\max}}}
    \|\phi(x_i)-\phi(x_j)\|_2$
    \STATE $\mu^* \leftarrow
    \begin{cases}
    |S(c_{\max})|^{-1}\sum_{z_i\in S(c_{\max})}\phi(x_i), & |S(c_{\max})|>0,\\
    |S|^{-1}\sum_{z_i\in S}\phi(x_i), & |S(c_{\max})|=0
    \end{cases}$
    \STATE $u^* \leftarrow \arg\max_{u\in\{i^*,j^*\}}\|\phi(x_u)-\mu^*\|_2$
    \STATE $L \leftarrow L\setminus\{z_{u^*}\}$
\ENDIF
\STATE \textbf{return} $S,L$
\end{algorithmic}
\end{algorithm}

The nearest-pair search implements the redundancy signal: among the overrepresented class in the long bank, the closest same-class pair is treated as the most redundant evidence, and the item farther from the recent reference centroid is removed.
\section{Experimental Details}
\label{app:experimental_details}

This appendix provides additional details on datasets, TFM
backbones, classical stream-learning baselines, context-management policy variants, and evaluation settings.

\subsection{Datasets}
\label{app:datasets}

\begin{table}[h]
\centering
\caption{Dataset statistics after applying the current stream loader.}
\label{tab:dataset_stats}
\begin{tabular}{lrrrl}
\toprule
Dataset & Rows & Features & Classes & Balance \\
\midrule
NOAA & 18,159 & 8 & 2 & imbalanced \\
METER & 22,948 & 96 & 10 & balanced \\
RIALTO & 82,250 & 27 & 10 & balanced \\
POSTURE-No8 & 163,477 & 3 & 10 & highly imbalanced \\
POKER & 829,201 & 10 & 10 & extremely imbalanced \\
NOMAO & 34,465 & 118 & 2 & imbalanced \\
AGR(A) & 30,000 & 9 & 2 & synthetic abrupt drift \\
\bottomrule
\end{tabular}
\end{table}

We evaluate seven datasets in total with six real-world data streams and one synthetic abrupt-drift stream. For real-world streams, we preserve the original row order and do not shuffle examples. Each file is loaded as a headerless CSV file, then the last column is treated as the class label while all preceding columns are used as input features. The \textsc{METER} and \textsc{POSTURE} loaders drop one malformed row before constructing the stream. A synthetic stream AGR(A) is generated from the Agrawal stream generator. We generated 30,000 examples with 9 features and 2 classes. The concept function changes abruptly at steps 7,500, 15,000, and 22,500 following the sequence $(0,3,6,9)$, with $10\%$ feature perturbation.

Table~\ref{tab:dataset_stats} summarizes the datasets used in our
experiments. We categorize class balance using the imbalance ratio
\[
    \mathrm{IR}
    =
    \frac{\max_c n_c}{\min_c n_c}.
\]
A dataset is categorized as balanced if $\mathrm{IR}\le 1.1$,
imbalanced if $1.1 < \mathrm{IR}\le 10$, highly imbalanced if
$10 < \mathrm{IR}\le 100$, and extremely imbalanced if
$\mathrm{IR}>100$.

\subsection{Backbones}
\label{app:backbones}

We evaluate \method with four frozen tabular foundation model
backbones: TabICL-v2, TabPFN-v2.5, TabDPT-v1, and LimiX-v1 to test the backbone-agnostic properties of \method.

\paragraph{TabICL-v2.}
TabICL was designed to scale tabular in-context learning beyond the expensive alternating row/column attention, using a two-stage column-then-row attention module to build row embeddings before efficient in-context prediction~\citep{qu2025tabicl}. TabICL-v2 further improves scalability and performance with a more diverse synthetic data engine and optimized pretraining protocols for both classification and regression~\citep{qu2026tabiclv2}. We use
TabICL-v2 as the primary backbone in the main experiments.

\paragraph{TabPFN-v2.5.}
TabPFN-v1 instantiated Prior-Data Fitted Networks for tabular classification. TabPFN-v2 introduced a stronger synthetic task distribution and an alternating-attention architecture over samples and features \citep{hollmann2025accurate}. TabPFN-v2.5 follows this line and scales it toward larger tabular contexts~\citep{grinsztajn2025tabpfn}.

\paragraph{TabDPT-v1.}
TabDPT targets the limitation of purely synthetic pretraining by training
tabular in-context learning architectures on real data with
self-supervised learning and retrieval~\citep{ma2024tabdpt}. We use TabDPT-v1 as the real-data-pretrained
backbone.

\paragraph{LimiX-v1.}
LimiX treats structured data as a joint distribution over variables and missingness, enabling classification, regression, imputation, and generation through query-based conditional prediction~\citep{zhang2025limix}. It is pretrained with masked joint-distribution modeling under an episodic context-conditional objective. We use the LimiX-16M checkpoint and the default no-retrieval classification configuration.

\subsection{Baselines}
\label{app:baselines}

We compare \method against various classical online stream-learning baselines. These baselines are implemented using MOA 24.07.

\paragraph{VFDT.}
The Very Fast Decision Tree (VFDT), also known as the Hoeffding Tree, is a milestone incremental decision-tree algorithm for high-speed data streams. It uses Hoeffding bounds to decide when the best split attribute is statistically reliable, enabling decision-tree induction with constant memory and constant time per example in the idealized setting
\citep{domingos2000mining}.

\paragraph{EFDT.}
The Extremely Fast Decision Tree (EFDT), or Hoeffding Anytime Tree, modifies Hoeffding Tree by allowing splits to be made earlier and later revised when better split choices become statistically supported. Compared with the conservative split policy of VFDT, EFDT is more statistically efficient and often obtains stronger prequential accuracy at modest additional computational cost~\citep{manapragada2018extremely}.

\paragraph{ARF.}
Adaptive Random Forest (ARF) is a strong ensemble baseline for evolving data stream classification. It combines online resampling with multiple incremental tree learners and uses adaptive operators to replace underperforming trees under concept drift~\citep{gomes2017adaptive}.

\paragraph{SRP.}
Streaming Random Patches (SRP) is an ensemble method designed for evolving streams that combines online bagging with random subspaces. Unlike methods that only randomize samples or only randomize features, SRP can jointly exploit instance resampling and feature-subspace diversity~\citep{gomes2019streaming}.

\paragraph{Leveraging Bagging.}
Leveraging Bagging extends online bagging with stronger randomization to increase ensemble diversity. It was proposed for evolving data streams as a simple but effective bagging variant with additional randomness and drift-aware mechanisms~\citep{bifet2010leveraging}.

\paragraph{BOLE.}
The Boosting-like Online Learning Ensemble (BOLE) adapts boosting-style ideas to concept-drifting data streams. It modifies online boosting mechanisms to better handle changing distributions and maintain ensemble performance under drift~\citep{de2016boosting}.

\subsection{Evaluation Details}
\label{app:evaluation_details}
\paragraph{Protocol.}
All methods are evaluated under prequential protocol. At step $t$, the learner predicts the label of $x_t$ using its current model or context, then receives the true label $y_t$, and finally updates its model or context using the labeled example $z_t=(x_t,y_t)$. We use a warm-up period of 100 stream steps and report prequential accuracy only after
warm-up.

\paragraph{Hardware.}
All TFM-based streaming experiments are run on a single NVIDIA H200 GPU.

\paragraph{MOA baseline hyperparameter grid.}
For all MOA baseline runs, we fix the split confidence to $10^{-7}$ and use \texttt{NBAdaptive} leaf prediction.

For ensemble baselines, ARF, BOLE, LevBag, and SRP, we fix the ensemble size to $90$ and grid search over grace period and tie threshold:
\[
    g\in\{100,400,1000\},
    \qquad
    t\in\{0.01,0.05,0.1\}.
\]
Thus, each ensemble baseline is evaluated with $3\times 3=9$
configurations per dataset. For tree baselines, EFDT and VFDT, we use the same grid over grace period and tie threshold:
\[
    g\in\{100,400,1000\},
    \qquad
    t\in\{0.01,0.05,0.1\}.
\]
This also gives $9$ configurations per tree baseline. 

\paragraph{Selected Hyperparameters}

Table~\ref{tab:selected_hyperparameters} reports the hyperparameters selected for the main results in Table~\ref{tab:main_results}. For \method, the selected value is the entropy threshold $\tau$. For the MOA baselines, $s$ denotes the ensemble size, $g$ denotes the grace period, and $t$ denotes the tie threshold. The ensemble size is fixed to $s=90$ for ARF, BOLE, LevBag, and SRP. EFDT and VFDT are single-tree methods and therefore do not use $s$.

\begin{table}[t]
\centering
\caption{Selected hyperparameters for the main comparison. For MOA baselines, $s$ is ensemble size, $g$ is grace period, and $t$ is tie threshold.}
\label{tab:selected_hyperparameters}
\vspace{0.25em}
\scriptsize
\setlength{\tabcolsep}{3pt}
\renewcommand{\arraystretch}{0.95}
\resizebox{\linewidth}{!}{
\begin{tabular}{lccccccc}
\toprule
Dataset & \method $\tau$ & ARF & BOLE & LevBag & SRP & EFDT & VFDT \\
\midrule
NOAA
& 0.4
& $s90,g400,t0.1$
& $s90,g100,t0.1$
& $s90,g100,t0.1$
& $s90,g400,t0.1$
& $g400,t0.1$
& $g400,t0.1$ \\

METER
& 0.1
& $s90,g400,t0.1$
& $s90,g1000,t0.1$
& $s90,g100,t0.1$
& $s90,g100,t0.1$
& $g1000,t0.01$
& $g400,t0.1$ \\

RIALTO
& 0.1
& $s90,g100,t0.1$
& $s90,g100,t0.1$
& $s90,g100,t0.1$
& $s90,g100,t0.1$
& $g100,t0.1$
& $g400,t0.1$ \\

POSTURE-No8
& 0.2
& $s90,g100,t0.1$
& $s90,g100,t0.1$
& $s90,g100,t0.1$
& $s90,g100,t0.1$
& $g400,t0.1$
& $g100,t0.1$ \\

POKER
& 0.3
& $s90,g100,t0.1$
& $s90,g100,t0.1$
& $s90,g100,t0.1$
& $s90,g100,t0.1$
& $g100,t0.1$
& $g400,t0.1$ \\

NOMAO
& 0.3
& $s90,g100,t0.1$
& $s90,g1000,t0.01$
& $s90,g400,t0.1$
& $s90,g100,t0.1$
& $g100,t0.05$
& $g100,t0.1$ \\

AGR(A)
& 0.4
& $s90,g100,t0.1$
& $s90,g100,t0.1$
& $s90,g100,t0.1$
& $s90,g100,t0.01$
& $g100,t0.1$
& $g400,t0.1$ \\
\bottomrule
\end{tabular}
}
\vspace{-0.5em}
\end{table}
\section{Additional Visualizations}
\label{app:visualizations}

\subsection{Prequential Accuracy Trajectories}
\label{app:prequential_trajectories}

Figures~\ref{fig:noaa_cure_vs_moa}--\ref{fig:agr_a_cure_vs_moa} show prequential accuracy trajectories of CURE and the MOA baselines and verify that CURE's advantage is persistent over time. On METER, RIALTO, POSTURE-No8, and POKER, CURE quickly separates from the classical baselines and maintains a large gap throughout most of the stream. On NOAA, NOMAO, and AGR(A), where the best classical baselines are closer, CURE still remains among the top trajectories and avoids the stronger degradation observed for some tree and ensemble methods. This supports the interpretation that the gains in Table~\ref{tab:main_results} reflect stable stream-level behavior rather than a final-score artifact.

\begin{figure}[h]
    \centering
    \includegraphics[width=\linewidth]{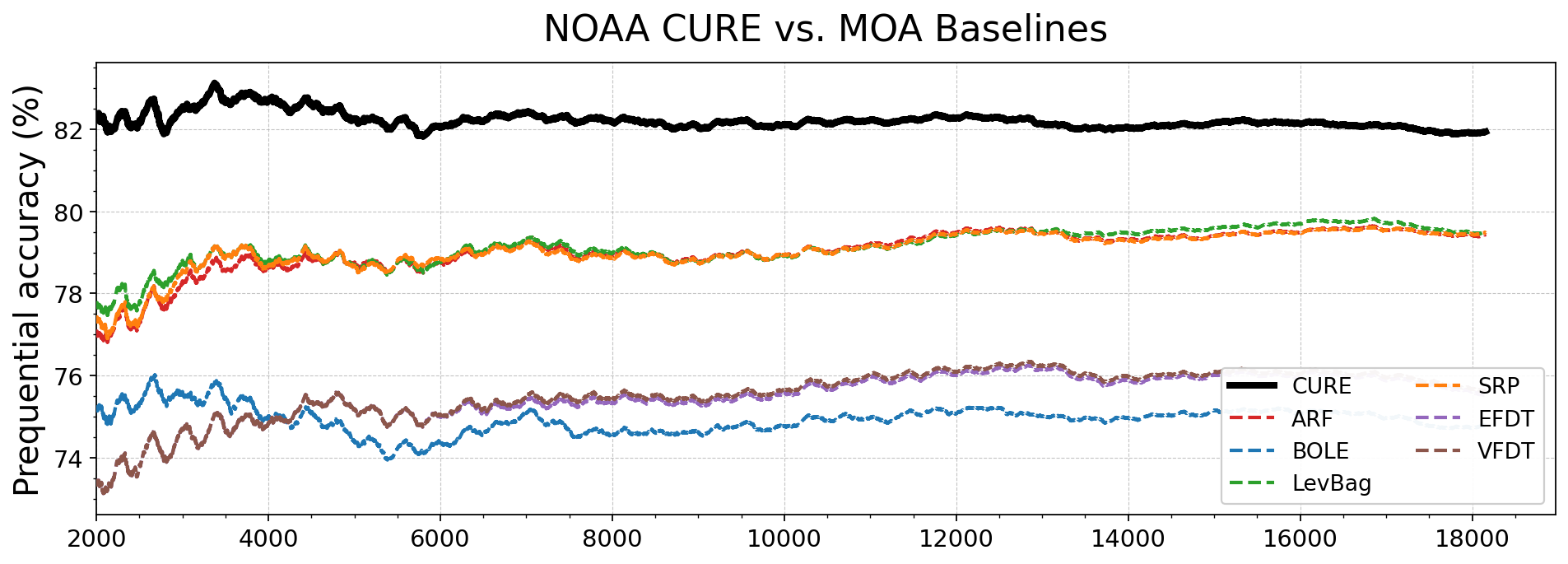}
    \caption{Prequential accuracy trajectories of CURE and MOA baselines on NOAA.}
    \label{fig:noaa_cure_vs_moa}
\end{figure}

\begin{figure}[h]
    \centering
    \includegraphics[width=\linewidth]{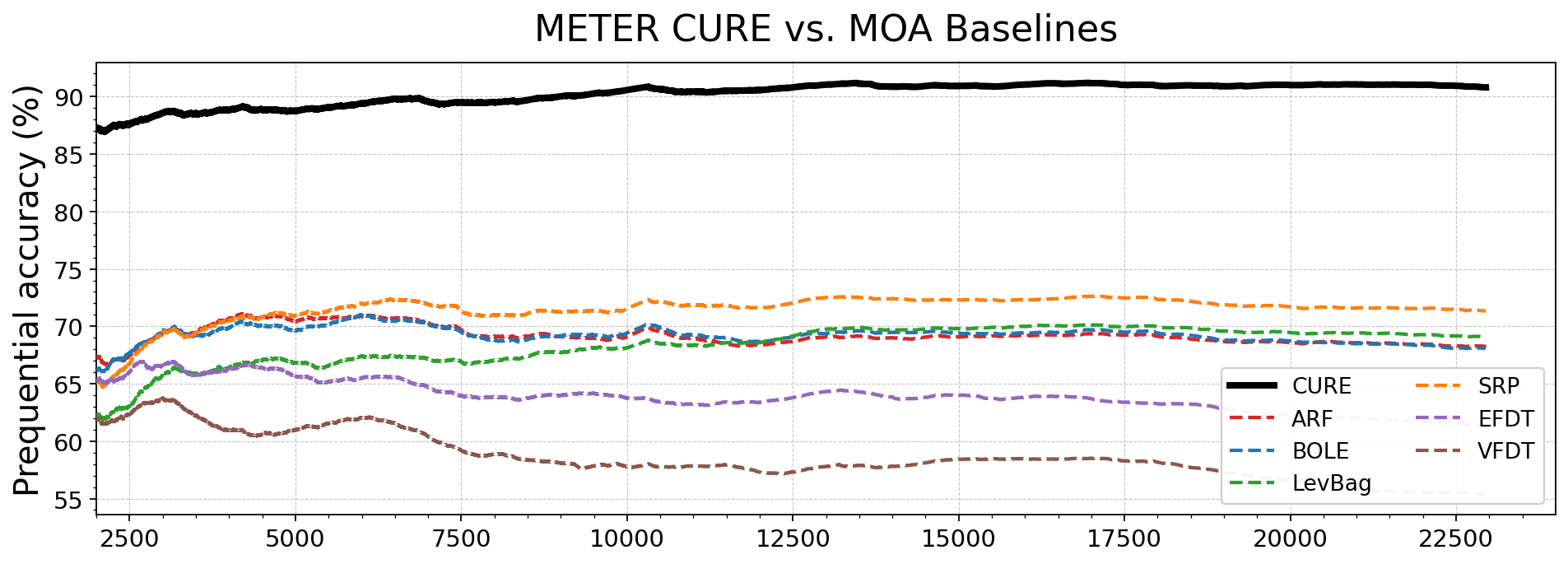}
    \caption{Prequential accuracy trajectories of CURE and MOA baselines on METER.}
    \label{fig:meter_cure_vs_moa}
\end{figure}

\begin{figure}[h]
    \centering
    \includegraphics[width=\linewidth]{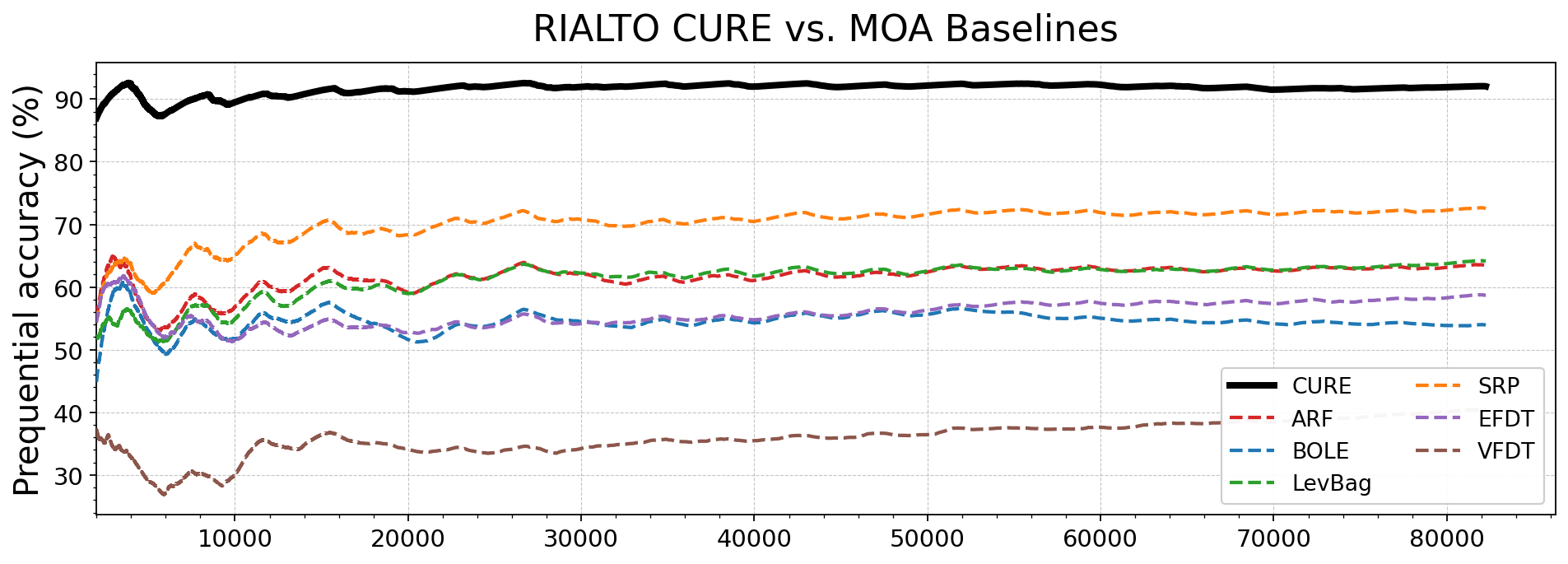}
    \caption{Prequential accuracy trajectories of CURE and MOA baselines on RIALTO.}
    \label{fig:rialto_cure_vs_moa}
\end{figure}

\begin{figure}[h]
    \centering
    \includegraphics[width=\linewidth]{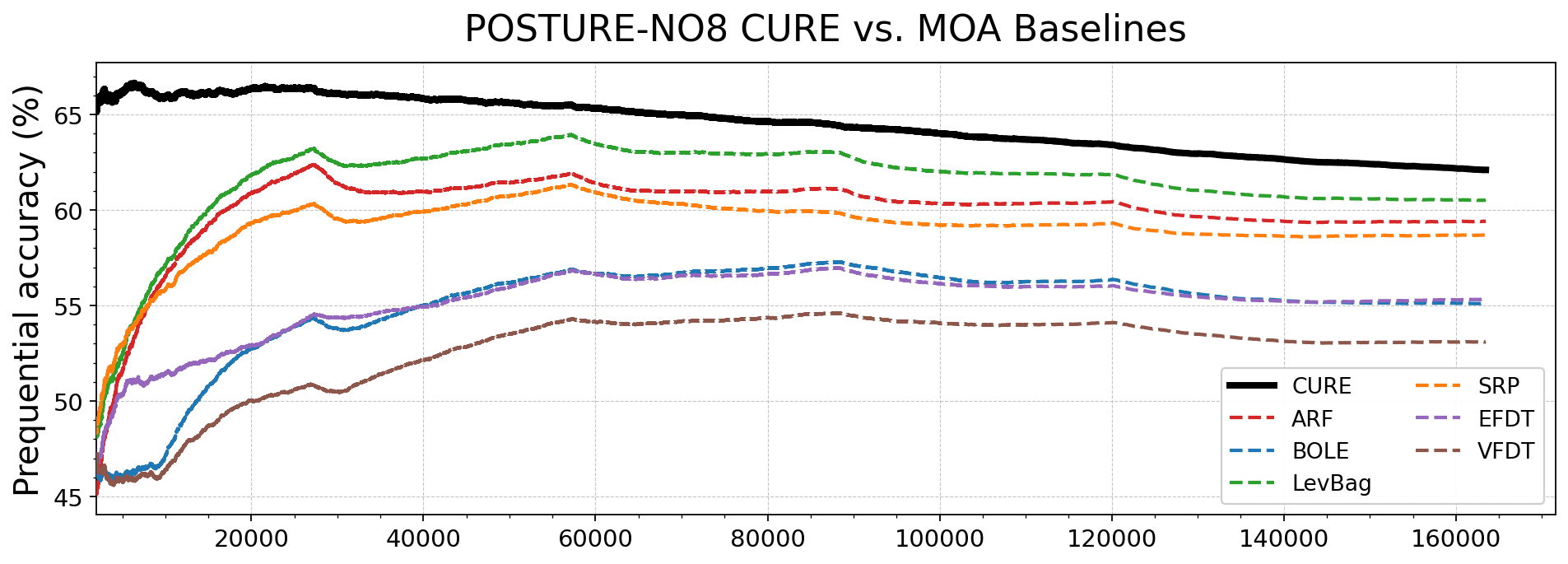}
    \caption{Prequential accuracy trajectories of CURE and MOA baselines on POSTURE-No8.}
    \label{fig:posture_no8_cure_vs_moa}
\end{figure}

\begin{figure}[h]
    \centering
    \includegraphics[width=\linewidth]{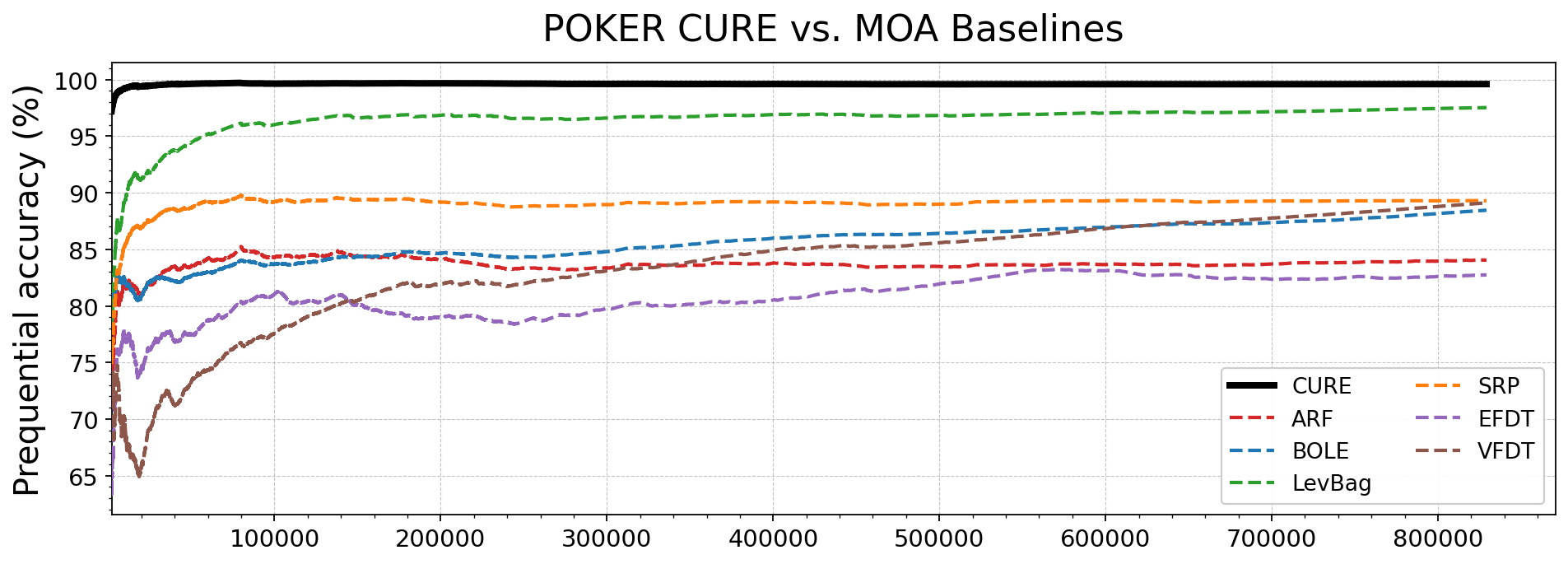}
    \caption{Prequential accuracy trajectories of CURE and MOA baselines on POKER.}
    \label{fig:poker_cure_vs_moa}
\end{figure}

\begin{figure}[h]
    \centering
    \includegraphics[width=\linewidth]{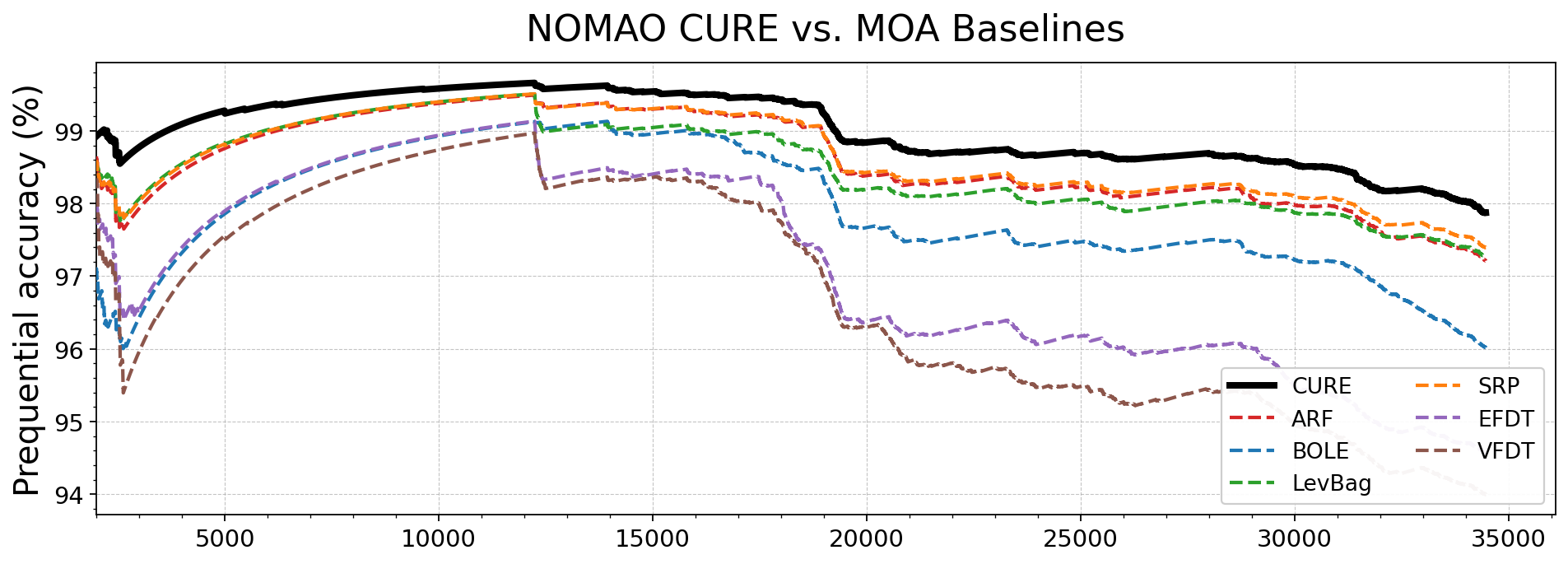}
    \caption{Prequential accuracy trajectories of CURE and MOA baselines on NOMAO.}
    \label{fig:nomao_cure_vs_moa}
\end{figure}

\begin{figure}[h]
    \centering
    \includegraphics[width=\linewidth]{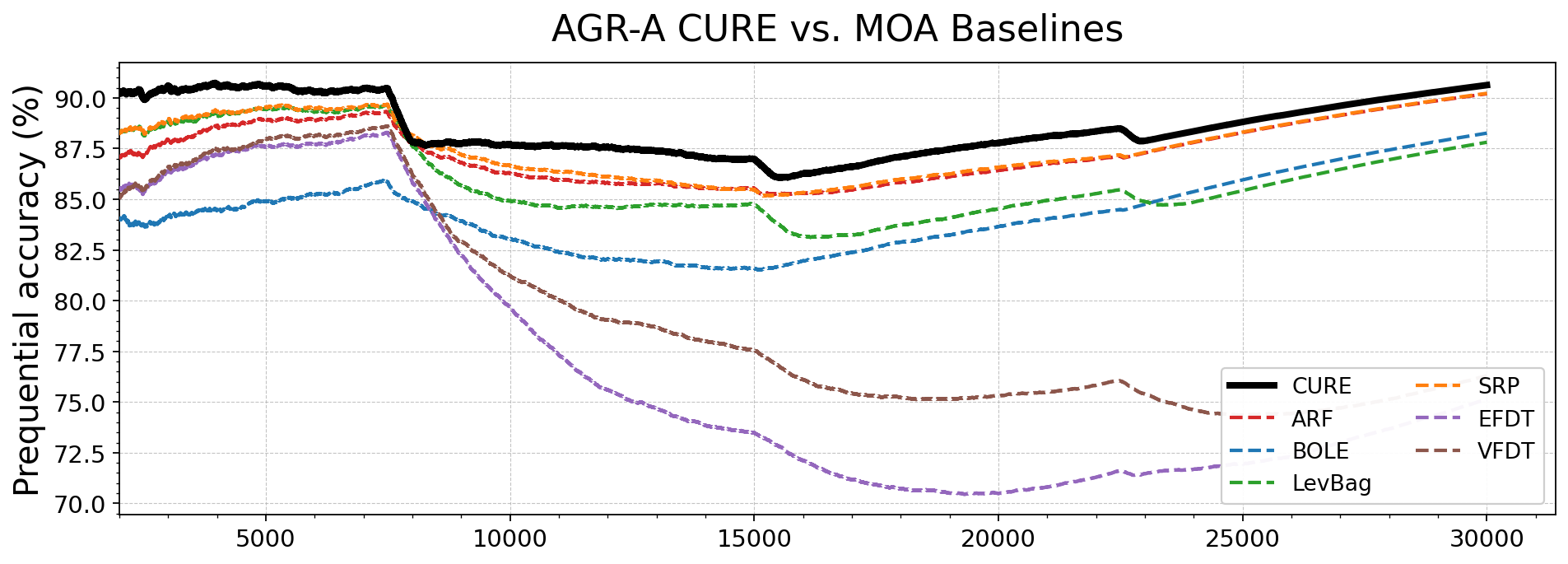}
    \caption{Prequential accuracy trajectories of CURE and MOA baselines on AGR(A).}
    \label{fig:agr_a_cure_vs_moa}
\end{figure}


\end{document}